\documentclass[11pt]{article}    
\usepackage{tgtermes}
\usepackage{enumerate}
\usepackage{bbm}
\usepackage[titletoc]{appendix}
\usepackage{fancyhdr}
\usepackage{pdfpages}\graphicspath{{/Users/mnk/Pictures/}}
\fancyheadoffset[LE,RO]{\marginparsep+\marginparwidth}
\usepackage{amsmath}
\usepackage{amssymb}
\usepackage{mathptmx}      
\usepackage{color}
\usepackage{mathdots}
\usepackage{dsfont}
\usepackage{pdfsync}
\usepackage{enumitem}
\usepackage{mathtools}
\usepackage{subfig}
\usepackage{endnotes}
\usepackage[numbers]{natbib}
\usepackage{hyperref}
\newcommand\bolds[1]{\textnormal{ \textbf{  #1}}}

\newcommand{\enumr}{\begin{enumerate}[label=\roman{*})]}
\newcommand{\enumR}{\begin{enumerate}[label=\Roman{*})]}
\newcommand{\enuma}{\begin{enumerate}[label=\alph{*})]}
\renewcommand{\leq}{\leqslant}

\newcommand{\pae}{\mbox{($\mathbb{P}$-a.e.)}}

\renewcommand{\geq}{\geqslant}
\renewcommand{\ge}{\geqslant}
\newcommand{\qed}{\hfill $\square$}

\newcommand{\cc}{\citet}

\newcommand{\esssup}{\mbox{\,\rm ess\,sup}} 
\newcommand{\argmax}{\mbox{\,\rm arg\,max}}

\newcommand{\eqqref}[1]{Eq.\,\eqref{#1}}

\newtheorem{lemma}{Lemma}
\newtheorem{theorem}{Theorem}
\newtheorem{proposition}{Proposition}
\newtheorem{definition}{Definition}

\newtheorem{corollary}{Corollary}

\newcounter{remno} 
\newenvironment{remark}{
	\refstepcounter{remno}
	\noindent {\bf Remark \theremno .} }%
{}
\newenvironment{proof}{\noindent {\bf Proof. }}{\hfill  \\}
\newfont{\mfoo}{cmssdc10 scaled\magstep1}
\newfont{\mfo}{cmtt9 scaled\magstep1}

  \newcounter{saa}     
  \newcommand{\pbsaa}{{\bf \arabic{saa}.\ }\stepcounter{saa}}
 \newcommand{\rbsaa}{\setcounter{saa}{1}{\bf \arabic{saa}.\ }\stepcounter{saa}}
\title{Optimal Activation of Halting Multi-Armed Bandit Models}
\small
\author{{\bf 
Wesley Cowan} \\
   Computer Science Department, Rutgers University\\
110 Frelinghuysen Rd., Piscataway, NJ 08854 
\and {\bf Michael N. Katehakis}\\
Management Science and Information Systems Department, Rutgers University\\
 100 Rockafeller Road, Piscataway, NJ 08854, USA
\and {\bf Sheldon M. Ross}\\ Systems Engineering Department, University of Southern California\\
3715 McClintock Ave GER 240, Los Angeles, CA 90089
}
\date{\small February, 2023}
\begin{document}
\maketitle
\begin{abstract}

We study  new types of  dynamic allocation problems the {\sl Halting Bandit} models. As  an application, we  obtain new proofs for the classic Gittins index decomposition result cf. \cc{Gittins:79}, and  recent results of the authors in \cc{cowan2015multi}.
 
\end {abstract}
 

%
%
%
\section{Introduction}

We investigate a class of  {\sl Halting Bandit} models, 
 where at every time step a controller must choose which project out of a fixed collection to activate, and 
 at  some (stochastic) time, when sufficient time and effort has been invested in a given project or process, it will be completed or ``halt''.  Additionally, halting may be considered a catastrophic event, such as a project breaking down.    
 These halting events allow bandits to be `singled out' - receiving rewards from successful bandits and paying costs for unsuccessful bandits. This singling out of projects based on state status is novel; prior results focused mainly  on maximizing cumulative collective payouts cf.  model (CCP) of Section 5. 
 

In this paper we  consider the following  models for  maximizing 
  terminal rewards (or minimizing terminal costs):    two versions of  expected 
{\sl terminal solo payout}, taken to be a reward dependent on the last  (ultimate)  or second to last (penultimate) state of the first bandit to halt successfully; the  
{\sl terminal collective payout} reward, taken to be a reward dependent on the final states of {\sl all} bandits at the first halting; the {terminal \sl non-halting costs}, taken to be a cost incurred by all bandits that {\sl failed} to halt; the {terminal \sl collective profit}, taken to be a reward from the successfully halted bandit less the cost incurred by bandits that failed to halt.
%
%
%
After establishing these results, we consider the same model in the framework of cumulative rewards, rather than terminal, when bandits are taken to generate rewards each time they are activated until halting. We use a standard technique 
to reduce  these models  to corresponding  terminal halting models and 
in this way, we recover prior results in \cc{cowan2015multi} and hence the celebrated Gittins' decomposition cf. \cc{Gittins:79}.

The central results presented here, the derivation of optimal policies for the terminal solo payout and terminal collective payout models, rests on establishing a correspondence between the two payout models; essentially, the game where every bandit contributes to the total reward can be replaced by an equivalent game where only a single bandit contributes to the terminal reward. This gives further insight into why classical bandit decomposition results work cf. \cc{Mahajan2014}, \cc{gittins2011multi},  \cc{Mah08}, \cc{ishikida1994multi}, \cc{Weber:92}.

%
%

For related work we first note that for the finite state Markov Chain  version of the cumulative collective payouts  model   of Section 5,  \cc{sonin2008generalized} introduced an equivalent  formulation of the indices derived herein in order to derive  
   an efficient   algorithm for the calculation of the   indices for all states of the Markov chain. 
   The basic idea of this paper's {\sl generalized indices} was to use 
 a common  Markov Decision Processes theory interpretation of of the expected discounted  total reward   with a discount factor $\beta$    where the state space is complemented by an absorbing state $x^*$ and  new transition probabilities that are defined as follows. The probability of entering an absorbing state $x^*$   in one step for any state 
 $y \neq  x^*$  (`probability of termination') is equal to $1-\beta$, and all other initial transition probabilities are multiplied by $\beta$. In other words, $\beta$ is the probability of `survival', or not 'halting' herein. \cc{sonin2008generalized} considered   variable probabilities of survival $\beta (x)$ and defined a generalized index   $\alpha (x)$   taken to be  
 the maximum ratio of the expected discounted total reward up to the time $\tau$ of halting (`termination')  per chance of termination at the time $\tau$ of halting.  He established that for non constant discount factors the 
 the equality  of the new {\sl generalized index} with the {\sl retirement index} of \cc{Whi80} and the restart index of \cc{Kat87}, thus he argued that the true meaning  of the Gittins index is given by its expression    as a ratio of 
     the expected discounted total reward up to the time $\tau$ of halting (`termination')  per chance of termination at the time $\tau$ of halting, and pointed out its relation with the work in
\cc{mitten1960analytic}. These results can be   extended  along the lines of \cc{el1994dynamic}
 who established the {\sl restart} representation of the Gittins index in a continuous time framework  without making further use of it. Additional  results connecting the Sonin indices with other problems of stochastic optimization are given in 
\cc{bank2004stochastic} and in \cc{sonin2016continue}.
   
For other related work we refer  to  \cc{szepesvari2010algorithms}, \cc{slivkins2019introduction}, 
\cc{dumitriu2003}, \cc{katta2004note}, and to 
\cc{stadje1995selecting},  \cc{pinedo1988note}, \cc{pinedo2012scheduling}, 
\cc{glazebrook2007index}, \cc{negoescu2018dynamic},    \cc{villar2015multi}   
\cc{glazebrook2014s}, 
 \cc{denardo2007risk},  \cc{KatR96}, \cc{katehakis1986computing}, and \cc{skitsas2022sifter}, \cc{talebi2021improved}, 
\cc{chk2018}.



The paper presents a collection of results, organized sequentially to build off each other to the final result. It is worth outlining this explicitly at the start, with a roadmap:

Section 2 gives the underlying mathematical framework of the discussion to follow, to guarantee the necessary processes and control processes are well defined. Ultimately, the key point of these results is this:  The relation  between      the `single payout' model and  the  `collective payout' model reveals   why the contributions of each bandit in the original formulation can be considered individually, by expressing the total game in terms of an equivalent one where only one bandit gives rewards. 
  Section 3 considers a simplified or `solo-payout' model, where only the bandit that halts (or breaks) yields a reward to the controller. These solo payout model bandits have a simple optimal policy.
In Section 4, we consider a collective-payout model (rewards from all bandits) and derive equivalent (or bounding) solo-payout models.
 The optimal solo-payout policy on the equivalent (or bounding) model is then shown to give an equivalent reward to a simple index policy on the collective-payout model, yielding a proof of optimality.
In Section 5, a number of alternative payout models are introduced, and all are shown to be equivalent to the solved collective-payout model. The classical Gittins formulation is recovered herein. Some proofs, technical and uninstructive, are relegated to Section 6.
 

\section{Problem Formulation}  
\subsection{Probability Framework}
A controller is presented with a finite collection of $N \geq 2$  probability spaces, $(\Omega^i, \mathcal{F}^i, \mathbb{P}^i, \mathbb{F}^i)$, for $1 \leq i \leq N$, representing $N$ environments in which experiments will be performed or rewards collected - the ``bandits,'' or ``projects.'' To each space, we associate an $\mathbb{F}^i$-adapted \emph{reward process} $X^i = \{ X^i_t \}_{t \geq 0}$. For $ t \in \{0, 1, \ldots \}$, we take $X^i_t (=X^i_t(\omega^i))  \in \mathbb{R}$ to represent the reward   (or state) attained from the $i^{th}$ bandit on its $t^{th}$ activation.
We denote the collection of these   processes as $\mathbb{X}$.

Additionally, to each bandit, we associate an $\mathbb{F}^i$-stopping time $\sigma^i > 0$, the ``halting time'' of the bandit, so that at the $\sigma^i$-th activation of bandit $i$, we take the bandit to be stopped, and no longer capable of being activated. Note, $\sigma^i$ represents the number of times bandit $i$ can be activated, so the last activation of bandit $i$ occurs at bandit-time $\sigma^i-1$, and at bandit time $\sigma^i$, the bandit is permanently stopped. On every  activation  prior to halting, we assume  there is a positive probability of halting. We take the first of any bandit halting to halt the entire decision process (game).

In what follows, we reserve the term ``round'' to differentiate global controller time (denoted with $s$), when the controller must decide which bandit to activate, from local bandit times (denoted by $t$), indicating the current total activations of a given bandit. In each round, the controller activates a bandit, advancing both its local time and the global time by one time step. All bandits begin at local time $0$, and advance only on activation, i.e.,  in every round \emph{unactivated bandits remain frozen}. As stated, the game halts upon the first halting of any bandit.
The controller needs a {\sl control policy} $\pi$,   that specifies, at each round $s$ of global time, which bandit to activate.

We embed these bandits in a larger product  space $(\Omega, \mathcal{G}, \mathbb{P}) =
( \otimes_{i = 1}^N \Omega^i, \otimes_{i = 1}^N \mathcal{F}^i, \otimes_{i = 1}^N \mathbb{P}^i )$, 
a standard product-space construction, representing the environment of the controller - aware information from all bandits. This `global'  probability space is necessary for making sure processes at the controller level (e.g., the policy for bandit activation) are well defined. This construction captures the first key aspect of the model: \emph{the bandits are mutually independent} (e.g., $X^i, X^j$ are independent relative to $\mathbb{P}$ for $i \neq j$). Expectations relative to the local space, i.e., bandit $i$, will be denoted $\mathbb{E}^i$, while expectations relative to the global space are simply $\mathbb{E}$.

\begin{remark}\label{rem:extensions} We adopt the following notational liberty, allowing a random variable $Z$ defined on a local space $\Omega^i$ to also be considered as a random variable on the global space $\Omega$, taking $Z(\omega) = Z(\omega^i)$, where $\omega = (\omega^1, \ldots, \omega^N) \in \Omega.$ Via this extension, we may take expectations involving a process $X^i$, or $\mathbb{F}^i$-stopping times, relative to $\mathbb{P}$ or $\mathbb{P}^i$, without additional notational overhead.
\end{remark}

We make the following assumptions.

 {\bf Assumption 1:} For each bandit  $i$   
\begin{equation}\label{eqn:integrability}
\mathbb{E}^i \left[ \sup_{n \geq 0}\ \lvert X^i_n \rvert \right] < \infty.
\end{equation}

 {\bf Assumption 2:} For each bandit  $i$ 
 the following are true.
\hspace{-2cm}
\begin{align}
a)     & \    &  &  \mathbb{P}^i(\sigma^i < \infty) = 1,  & \ \label{eqn:fsigma} \\
b)     &    &    & \mathbb{P}^i( \sigma^i = t + 1 | \mathcal{F}^i(t) ) > 0, \ \mbox{for all $t < \sigma^i$,}   \mbox{ ($\mathbb{P}^i,$ $ \mathbb{P}$-a.e.). }& \ \label{eqn:transitions}
\end{align}

\begin{remark}
Note, the above assumptions, while technical in statement, have natural interpretations: 2.a) each bandit will halt after finite activations, almost surely; 2.b) at any time prior to halting, there is non-zero probability of halting on the next activation.
\end{remark}

A {\sl control policy} $\pi$, is a stochastic process on $(\Omega, \mathcal{G}, \mathbb{P})$ that specifies, at each round $s$ of global time, which bandit to activate and collect from, e.g., $\pi(s)(=\pi(s, \omega)) = i$ activates bandit $i$ at round $s$. We restrict  attention  to the set of policies $\mathcal{P}$ defined to be {\sl non-anticipatory}, i.e.,    the choice of which bandit to activate at round $s$ does not depend on outcomes that have not yet occurred, or information not yet available.

A policy $\pi$ defines $T^i_\pi(s)$ the {\sl $\pi$-local time} of bandit $i$ just prior to the $s^{th}$ 
 round under it,   i.e.,  $T^i_\pi(0) = 0$, and for $s > 0$,
\begin{equation}
T^i_\pi(s) = \sum_{s' = 0}^{s-1} \mathds{1}\{ \pi(s') = i\} .
\end{equation}

Note, this gives as a result that at global time $s$, the sum of all the local times must be $s$, i.e.,
\begin{equation}
T^1_\pi(s) + T^2_\pi(s) + \ldots + T^N_\pi(s) = \sum_{i = 1}^N \sum_{s' = 0}^{s-1} \mathds{1}\{ \pi(s') = i\} = \sum_{s' = 0}^{s-1}  \sum_{i = 1}^N \mathds{1}\{ \pi(s') = i\} = \sum_{s' = 0}^{s-1} 1 = s,
\end{equation}
where the inner sum reduces to $1$ since exactly one bandit is activated each round.

It is convenient to define the global time analog, $T_\pi(s) = T^{\pi(s)}_\pi(s)$ to denote the current $\pi$-local time of the bandit activated at round $s$ under policy $\pi$. This will allow us to define concise global time analogs of several processes. An important example of such a process is  the {\sl global reward process} $X_\pi$ on $(\Omega, \mathcal{G}, \mathbb{P})$ defined as 
$$X_\pi(s) = X^{\pi(s)}_{ T_\pi(s) },$$ 
giving the reward available from collection $\mathbb{X}$ under policy $\pi$, which is to be received if the game halts  at round $s$.

To be  able to translate between global time and local times, when the controller operates according to a policy $\pi$, 
 we define the random variables  $S^i_\pi(t)$ to 
represent {\sl the round at which bandit $i$ is activated for the $t^{th}$ time},   i.e.,   
\begin{equation}\label{eqn:global-time}
\begin{split}
S^i_\pi(0) & = \inf \{ s \geq 0 : \pi(s) = i \}, \\
S^i_\pi(t + 1) & = \inf \{ s > S^i_\pi(t) : \pi(s) = i \}. \\ 
\end{split}
\end{equation}

Utilizing this notation, we may define a {\sl global halting time} $\sigma_\pi$, i.e., the first round under policy $\pi$ at which one of the bandits has halted, ending the game:
\begin{equation}
\sigma_\pi = \min_i\{ S^i_\pi(\sigma^i-1) \} + 1.
\end{equation}

\begin{remark}
To clarify the   above definition,  note that $S^i_\pi(0)$ is the time that a policy first activates bandit $i$, advancing it from local time $0$ to local time $1$. So $S^i_\pi(\sigma^i-1)$ is the global time round at which the policy $\pi$ advances bandit $i$ from local time $\sigma^i-1$ to local time $\sigma^i$, halting that bandit. The expression above for $\sigma_\pi$ therefore identifies the first global round at which no further activations will be made, because one bandit has been halted.
\end{remark}

In what follows, for a given policy $\pi$, we take the final reward the controller receives to be a function of the last rewards   of the game, generally a linear combination of $\{ X^i_{T^i_\pi(\sigma_\pi)} \}_{ 1 \leq i \leq N }$, or  in the penultimate model a function of the second to last rewards. To maximize her expected reward, in every round the controller's decision of which bandit to activate must balance not only the current rewards of each bandit, but also the probability of halting that bandit and in doing so ending the game - losing all potential future rewards.

\subsection{Global Information Versus  Local Information}\label{sec:information}

One of the intricacies of the results to follow is in properly distinguishing and determining what information is available to the controller to act on at a given time. The following statements are somewhat technical, but necessary for the purpose of making sure all relevant processes are mathematically well-defined, and that our control processes do not depend on information they should not have access to. Ultimately, the optimal policy results of Theorems 1 and 4 (essentially stating the simplicity of the optimal policy) demonstrate that in the optimal policy, any decision to activate a given bandit depends only on information from other bandits individually, thus rendering these filtrations unnecessary under an optimal policy. However, these  extended filtrations are a technical necessity for the proof of Theorems 4.

For each bandit $i$, the filtration $\mathbb{F}^i = \{ \mathcal{F}^i(t) \}_{t \geq 0}$ represents the progression of information available about that bandit - the $\sigma$-algebra $\mathcal{F}^i(t)$ representing the local information available about bandit $i$ at local time $t$, such as (but not limited to) the process history of $X^i$. Taking $X^i$ as $\mathbb{F}^i$-adapted as we do, we have $\sigma(X^i_0, X^i_1, \ldots, X^i_t) \subset \mathcal{F}^i(t)$.

At round $s$, all information available to the controller is determined by the state of each bandit at that round, i.e. acting under a given policy $\pi$ until round $s$, the global information available at round $s$ is given by the $\sigma$-algebra $\bigotimes_{i = 1}^N \mathcal{F}^i( T^i_\pi(s) )$. We may therefore refine the prior definition of non-anticipatory policies to be the set of policies $\mathcal{P}$ such that for each $s \geq 0$, $\pi(s)$ is measurable with respect to the prior $\sigma$-algebra, i.e., determined by the information available at round $s$. Weaker definitions of non-anticipatory, such as dependence on random events, e.g., coin flips, are addressed in Section 6. It is convenient to define the initial global $\sigma$-algebra $\mathcal{G}_0 = \bigotimes_{i = 1}^N \mathcal{F}^i( 0 )$, representing the initial information available from each bandit, which is independent of policy $\pi$.

Additionally, given a policy $\pi$, it is necessary to define a set of policy-dependent filtrations in the following way: let $\mathbb{H}^i_\pi = \{ \mathcal{H}^i_\pi(t) \}_{t \geq 0}$, where $\mathcal{H}^i_\pi(t) = \bigotimes_{j = 1}^N \mathcal{F}^j( T^j_\pi(S^i_\pi(t)) )$ represents the total information available to the controller about all bandits, prior to the $t^{th}$ activation of bandit $i$ under $\pi$. It is indexed by the local time of bandit $i$, but at each time $t$ gives the current state of information of each bandit. Note that, since $T^i_\pi( S^i_\pi(t) ) = t$, $\mathcal{H}^i_\pi(t)$ contains the information available in $\mathcal{F}^i(t)$. This filtration is necessary for expressing local stopping times, i.e., concerning $X^i$, from the perspective of the controller - $\mathbb{F}^i$-stopping times no longer suffice, since the controller has access to information from all the other processes as well. Note though, $\mathbb{F}^i$-stopping times may be viewed as $\mathbb{H}^i_\pi$-stopping times, cf. Remark  \ref{rem:extensions}. 

{\bf Notation.} 
When discussing stopping times, we will utilizing the following notation. For a general filtration $\mathbb{J}$ (e.g., $\mathbb{J}= \mathbb{F}^i, \mathbb{H}^i_\pi$), we denote by $\hat{\mathbb{J}}(t)$ the set of all $\mathbb{J}$-stopping times strictly greater than $t$ ($\mathbb{P}^i,\mathbb{P}$-a.e.). For a $\mathbb{J}$-stopping time $\tau$,  $\hat{\mathbb{J}}(\tau)$ is similarly defined.
 
The following simple example illustrates the random variables we have defined in this section.

\noindent\bolds{Example 1.}
Take  $N=2$  bandits,   independent geometric stopping times $\sigma^i $ with 
$$\mathbb{P}^i(\sigma^i > t)=\beta_i^{t},  \ \mbox{for $t=0,1,\ldots$} $$   for some constants 
 $\beta_i\in (0,1),$  $i=1,2$, and 
 consider a cyclic policy $\pi^1(t)=1$ for $t=0,2,\ldots,$  and $\pi^1(t)=2$ for $t=1,3,\ldots$\ . 
 Under the  policy $\pi^1$ for any sample path for which $\sigma^1>2$ and $\sigma^2>2$ we will have:

\begingroup
\setlength{\tabcolsep}{12pt} 
\renewcommand{\arraystretch}{1.6}
  \begin{tabular}{|l||c|c||c|| c|c|}
 \hline 
     $s$ &  $T^1_{\pi^1}$ &  $T^2_{\pi^1}$ & $\pi^1(s)$ & Reward & Probability of not  stopping at $s$\\
    \hline 
    0 & 0 & 0 & 1& $X^1_0$ &   $\beta_1=\mathbb{P}^i(\sigma^1  >  1)$\\
    \hline
    1 & 1 & 0 & 2& $X^2_0$ &$\beta_1 \beta_2=\mathbb{P}^i(\sigma^1  >  1, \sigma^2  >  1) $\\
    \hline
        2 & 1 & 1 & 1& $X^1_1$ & $ \beta_1^2 \beta_2=\mathbb{P}^i(\sigma^1  >  2, \sigma^2  >  1) $\\
    \hline
        3 & 2 & 1 & 2& $X^2_1$& $ \beta_1^2 \beta_2^2=\mathbb{P}^i(\sigma^1  >  2, \sigma^2  >  2) $\\
    \hline
       4 & 2 & 2 & 1& $X^1_2$& $ \beta_1^3 \beta_2^2=\mathbb{P}^i(\sigma^1  >  3, \sigma^2  >  2) $\\
    \hline
       $\vdots$ &  $\vdots$ &  $\vdots$ &  $\vdots$ &  $\vdots$&  $\vdots$\\
\end{tabular}\\
\endgroup

Thus is easy to see that under $\pi^1$ the expected   
total reward  received from the two bandits is  
 $$
V_{\pi^1}(\mathbb{X}) =  \mathbb{E} \left[ X^1_0 +\beta_1X^2_0 + \beta_1 \beta_2 X^1_1 + \beta_1^2 \beta_2 X^2_1 + \beta_1^2 \beta_2^2 X^1_2 +\cdots \right]. 
 $$
 Note also that:\\
  $S^1_{\pi^1}(0)=\inf\{s>0 :  \pi^1(s)=1\}=0, $ 
  $S^1_{\pi^1}(1)=\inf\{s>S^1_{\pi^1}(0)  :   \pi^1(s)=1\}=2, $  
   $S^1_{\pi^1}(2)=4$ 

 and 
  
   $S^2_{\pi^1}(0)=\inf\{s>0 :  \pi^1(s)=2\}=1, $ 
  $S^2_{\pi^1}(1)=\inf\{s>S^2_{\pi^1}(0)  :   \pi^1(s)=2\}=3, $  etc. 
  
  Finally note that under policy $\pi^1$  on the event $\{\sigma^1 \ge 2, \ \sigma^2=1\}$ 
  bandit $2$ causes the game to end at round $s=2$ i.e., 
  the {\sl global halting time} is   
$$\sigma_{\pi^1} = \min \{ S^1_{\pi^1}(\sigma^1-1), \ S^2_{\pi^1}(0) \} + 1=1+1=2,$$ 
since $S^1_{\pi^1}(\sigma^1-1)\ge S^1_{\pi^1}(1)=2 .$ 

\section{Maximizing Solo Payouts: Non-Increasing Rewards}\label{sec:monotone}
In this section, we consider the problem of maximizing the expected \emph{penultimate} reward from the bandit that halts and ends the game. That is, if a bandit is activated and halts, stopping the game, the controller receives the reward that bandit offered prior to its last activation.  Additionally,  in this section, we assume that the reward processes of each bandit are non-increasing. In fact, under this restriction, we may even maximize the reward \emph{almost surely}. This result, while intuitive, acts as the basis of 
all  future optimality results herein.

We define the \emph{penultimate solo payout} value of a policy $\pi$ as, 
\begin{equation}
\begin{split}
V^{PSP}_\pi(\mathbb{X}) & = \mathbb{E} \left[ X_\pi(\sigma_\pi-1) | \mathcal{G}_0 \right]\\
& = \sum_{i = 1}^N \mathbb{E} \left[ \mathbbm{1}\{ i = \pi(\sigma_\pi-1) \} X^i_{T^i_\pi(\sigma_\pi-1)} | \mathcal{G}_0\right].
\end{split}
\end{equation}

\begin{theorem}[A Greedy, Almost-Sure Result for Non-Increasing Solo Payout Processes]\label{thm:greedy}
Given a collection of reward processes $\mathbb{X}$ such that for each $i$, $X^i$ is almost surely non-increasing for $t < \sigma^i$, there exists a policy $\pi^* \in \mathcal{P}$ such that for any policy $\pi \in \mathcal{P}$,
\begin{equation}\label{eqn:greedy-rewards}
X_\pi(\sigma_\pi-1) \leq X_{\pi^*}(\sigma_{\pi^*} - 1) \mbox{\ \  ($\mathbb{P}$-a.e.)}.
\end{equation}

In particular, such a $\pi^*$ is given by the following greedy rule: In each round $s \geq 0$, activate the bandit with the largest current value of $X^i$, i.e.,  $$\pi^*(s) = \argmax_i\ X^i_{T^i_{\pi^*}(s)}.$$
\end{theorem}

\begin{proof}
The proof proceeds by incremental improvements on an arbitrary policy.

Let $X^i_0 = \max_j X^j_0$. Let $\pi \in \mathcal{P}$ be arbitrary, and define $S = S^i_\pi(0)$, the first round bandit $i$ is activated under $\pi$. If $i$ is never activated, we take $S$ to be infinite.

From $\pi$, we construct a policy $\pi' \in \mathcal{P}$ as follows: $\pi'$ activates bandits in the same order as $\pi$, but it advances the first activation of bandit $i$ from round $s = S$ to round $s = 0$. That is,
\begin{equation}
\pi'(s) =
\begin{cases}
i &  \mbox{for } s = 0, \\
\pi(s-1) & \mbox{for }  s =1,2,\ldots S, \\
\pi(s) & \mbox{for } s \ge S+1 . \\
\end{cases}
\end{equation} 
That is, after the initial round policy $\pi'$ activates the bandit that policy $\pi$ activated in the previous round, continuing this through the first round that $\pi$ activates bandit $i$,  then making the same choice in each round as does $\pi.$ Policy $\pi'$ is well-defined and in $\mathcal{P}$, as at every round $s$, the information available under $\pi'$ about each bandit is greater than or equal to the information available under $\pi$ at that round.


We next compare the performance of these two policies by cases. 

In the case that $\sigma_\pi > S +1 $ ($= S^i_\pi(0) + 1)$, that is when the game 
halts under $\pi$  \textit{after} the first activation of bandit $i$, then there is no difference between the rewards returned by either policy, since both policies perform the same activations after time $S$ (sample path-wise). 

Similarly, if $\sigma_\pi = S + 1$, that is $\pi$ halts \textit{due to} the first activation of bandit $i$, the reward returned under $\pi$ is $X^i_0$, and as bandit $i$ halted on its first activation, the reward returned under $\pi'$ is also $X^i_0$. 

Finally   the only situation in which $\pi$ and $\pi'$ differ in their returned rewards is when $\sigma_\pi \leq S$ and $\sigma^i = 1$. 

Therefore, it follows from the above cases that:
\begin{equation}
\begin{split}
X_{\pi'}(\sigma_{\pi'}-1) - X_{\pi}(\sigma_{\pi}-1) & = (X_{\pi'}(\sigma_{\pi'}-1) - X_{\pi}(\sigma_{\pi}-1) )\mathbbm{1}_{ \{ \sigma_{\pi} \leq S \} }\mathbbm{1}_{ \{ \sigma^i = 1 \} }\\
& = (X^i_0 - X_{\pi}(\sigma_{\pi}-1) )\mathbbm{1}_{ \{ \sigma_{\pi} \leq S \} }\mathbbm{1}_{ \{ \sigma^i = 1 \} }\\
& \geq 0 \mbox{\ \  ($\mathbb{P}$-a.e.)}.
\end{split}
\end{equation}
The last step follows taking $X^i_0$ as the initial largest reward, and that all bandits are non-increasing.

It follows that advancing the activation of the initial maximal bandit improves or at least does not change the value of a policy. This same argument can be applied at every round that follows, i.e.,  at every round, activation of the current initial maximal bandit is an improvement over (or at least does not change the value) of any other policy. Note, collisions may occur if at a given round two bandits have equal rewards. This may be resolved at the discretion of the controller, such as by always taking the bandit with the smaller index $i$.

As each bandit halts in a finite time, almost surely, for sufficiently many greedy improvements as outlined above, the resulting improvement of any policy $\pi$ will return the same value as the completely greedy strategy $\pi^*$. Hence,
\begin{equation}
X_{\pi^*}(\sigma_{\pi^*} - 1) \ge X_\pi(\sigma_\pi-1) \mbox{\ \  ($\mathbb{P}$-a.e.)}.
\end{equation}\qed
\end{proof}
\begin{remark} {The Necessity of finite $\sigma^i$.}
Note that  Assumption 2.a:   $\sigma^i < \infty$ almost surely, for each bandit $i$, is employed to exclude cases such as the following, in which no optimal policy exists.

Consider two bandits, Bandit A offering a potential reward of $\$100$ in each time step, and Bandit B offering a potential reward of $\$50$ in each time step. Further, suppose that $\mathbb{P}^A(\sigma^A < \infty) = 0.5$, and $\sigma^B = 1$ almost surely - that is, Bandit B halts after its first activation.

This choice of $\sigma^B $ implies that any policy on these bandits may be described in the following way: For any a.s.  finite $\mathbb{F}^A$-stopping time $\tau \geq 0 ,$ $\pi_\tau$ activates Bandit A until $\tau$, then Bandit B, ending the game. The value of such a policy is given by
\begin{equation}
V^{PSP}_{\pi_\tau}(A,B) = \$100\ \mathbb{P}^A(\sigma^A < \tau) + \$50\ \mathbb{P}^A(\sigma^A \geq \tau)  \leq 75.
\end{equation}
This upper bound may be achieved within an arbitrary amount by choosing a finite, sufficiently large $\tau$ - the larger the $\tau$, the closer to achieving the upper bound. However, taking $\tau$ to be infinite, the $\$100$ is only collected with probability $0.5$, and Bandit B is never activated at all, yielding a total expected value of $\$100\times0.5 = \$50 < \$75.$ In this case, there exist $\epsilon$-optimal policies, but no optimal policy. This phenomenon appears in all versions of the problems discussed herein and its investigation is  an avenue of interesting additional  research. 

\end{remark}

\section{Maximizing Collective Payouts}\label{sec:collective}

In this section, we consider a model where rewards are collective i.e., received from all bandits, at the halting of the game. Thus, the expected  \emph{collective payout} value of a policy $\pi$ is  
\begin{equation}\label{eq:CP}
V^{CP}_\pi(\mathbb{X}) = \sum_{i = 1}^N \mathbb{E} \left[ X^i_{T^i_\pi(\sigma_\pi)} | \mathcal{G}_0\right].
\end{equation}
 In the following subsections, we develop a policy $\pi^* \in \mathcal{P}$ such that for all $\pi \in \mathcal{P}$,
\begin{equation}
V^{CP}_\pi(\mathbb{X}) \leq V^{CP}_{\pi^*}(\mathbb{X}) \mbox{ ($\mathbb{P}$-a.e.). }
\end{equation}

\begin{remark}For algebraic convenience in the remainder of this section  we take $X^i_0 = 0$ for all $i$. For a more arbitrary reward processes $\{ \hat{X}^i \}$, recall that the initial $\hat{X}^i_0$ are taken to be constant and known at the initial round by assumption. Hence, defining $X^i_t = \hat{X}^i_t - \hat{X}^i_0$, maximizing the total expected reward from the $\{ \hat{X}^i \}$ processes is equivalent to maximizing the total expected reward from the $\{ X^i \}$ processes.
\end{remark}

\subsection{Block Values}

This section introduces a way of considering the ``value'' of a set of activations of a bandit. The ``true'' value of a decision to activate a bandit is not simply the potential reward gained through that decision, but instead it must balance the immediate potential reward with the incurred risk  of halting  the game through that decision, and the resulting loss of potential future rewards.

For each bandit $i$, for a given policy $\pi$ we define $\tau^i_\pi$ to be the first activation of bandit $i$ that {\sl does not occur under $\pi$}. That is,
\begin{equation}
\tau^i_\pi = \min \{ t \geq 0 : S^i_\pi(t) \geq \sigma_\pi \}.
\end{equation}
Note, the above makes use in its definition of $\pi$ `after the halting time $\sigma_\pi$', but we simply mean to observe here that at the global halting time, we can observe what the next activation of each bandit would have been - this is $\tau^i_\pi$.

With this, we state the following definitions.

\begin{definition}[Process Blocks and their Values]\label{def:block-values}
Given times $t' < t''$ with $t' < \sigma^i$, and a policy $\pi \in \mathcal{P}$ with $S^i_\pi(t') < \sigma_\pi$:
\begin{enumerate}
\item The \emph{solo-payout value of the $[t', t'')$ - block of $X^i$} as:
\begin{equation}\label{eqn:block-value}
\rho^i(t', t'') = \frac{ \mathbb{E}^i \left[ X^i_{ \sigma^i \wedge t'' } - X^i_{t'} \big| \mathcal{F}^i(t') \right] }{ \mathbb{P}^i\left( t' < \sigma^i \leq t'' \big| \mathcal{F}^i(t') \right) }.
\end{equation}
\item The \emph{$\pi$-value of the $[t', t'')$ - block of $X^i$} as:
\begin{equation}\label{eqn:pi-value}
\nu^i_\pi(t', t'') = \frac{ \mathbb{E} \left[ X^i_{ T^i_\pi(\sigma_\pi) \wedge t'' } - X^i_{t'} \big| \mathcal{H}^i_\pi(t') \right] }{ \mathbb{P}\left( t' < \sigma^i \leq \tau^i_\pi \wedge t'' \big| \mathcal{H}^i_\pi(t') \right) }.
\end{equation}
\end{enumerate}
\end{definition}

\begin{remark} Due to Eq. \eqref{eqn:transitions} c.f. Assumption 2, the denominators of both block values are non-zero. The above quantities are all measurable with respect to the indicated $\sigma$-fields, and finite ($\mathbb{P}^i, \mathbb{P}$ -a.e.), due to \eqqref{eqn:integrability} c.f. Assumption 1.
\end{remark}

Notionally, $\rho^i$ can be thought of as the value of a block under consecutive activation, while $\nu^i_\pi$ is, correspondingly, the value of a block potentially `diluted' or broken up by activations of other bandits under $\pi$.
The denominator of $\nu^i_\pi$ may be interpreted as the probability that the game halts {\sl due to bandit $i$}, halting during activation of the $[t', t'')$-block.

\begin{remark}
The above might be justified as the `value' of a block of activations in the following way: even if the incremental reward gained due to an activation block (the numerators) is small, if the probability of halting due to those activations (the denominators) is sufficiently small, there is very little risk in attempting to gain that increment through that activation. In fact, there might be more to gain in such a case than if the incremental reward were slightly larger, but the probability of halting were also larger. The above values captures this trade-off between risk of halting and reward gained.
\end{remark}

 The following theorem illustrates the relationship between $\rho^i$ and $\nu^i_\pi$, essentially stating that the value of any block under some policy $\pi$ is {\sl at most} the value of {\sl some} block activated consecutively. 

\begin{theorem}[Block Value Comparison]\label{thm:a-value-comparison-theorem}
For bandit $i$ under policy $\pi$, for any time $t_0$ such that $S^i_\pi(t_0) < \sigma_\pi$, the following holds for any $\mathbb{H}^i_\pi$-stopping time $\tau$ with $t_0 < \tau$:

\begin{equation}\label{eqn:comparison-equation}
\nu^i_\pi(t_0, \tau) \leq \underset{\mbox{\ } \hat{\tau} \in \hat{\mathbb{F}}^i (t_0) }{\esssup}~\rho^i(t_0, \hat{\tau}) \mbox{ \emph{($\mathbb{P}$-a.e.)}}.
\end{equation}
 \end{theorem}
 \begin{proof}
Note that it follows from Eqs. (\ref{eqn:integrability}), (\ref{eqn:transitions}) that the essential supremum is finite $\mbox{\emph{($\mathbb{P}$-a.e)}}$.

For each bandit $i$ and any $\pi \in \mathcal{P}$, it can be shown by cases (whether the game does or does not halt due to an activation of $i$\,) that $T^i_\pi(\sigma_\pi) = \sigma^i \wedge \tau^i_\pi$.

Therefore, for a given $\tau \in \hat{\mathbb{H}}^i_\pi(t_0)$,
\begin{equation}\label{eqn:comparison-main-bound}
\begin{split}
\nu^i_\pi(t_0, \tau) & =  \frac{ \mathbb{E} \left[ X^i_{ \sigma^i \wedge \tau^i_\pi \wedge \tau } - X^i_{t_0} \big| \mathcal{H}^i_\pi(t_0) \right] }{ \mathbb{P}\left( t_0 < \sigma^i \leq \tau^i_\pi \wedge \tau \big| \mathcal{H}^i_\pi(t_0) \right) } \\
& =  \frac{ \mathbb{E} \left[ X^i_{ \sigma^i \wedge (\tau^i_\pi \wedge \tau) } - X^i_{t_0} \big| \mathcal{H}^i_\pi(t_0) \right] }{ \mathbb{P}\left( t_0 < \sigma^i \leq (\tau^i_\pi \wedge \tau) \big| \mathcal{H}^i_\pi(t_0) \right) }  \leq \underset{\mbox{\ } \hat{\tau} \in \hat{\mathbb{H}}_\pi^i (t_0) }{\esssup}~\frac{ \mathbb{E} \left[ X^i_{ \sigma^i \wedge \hat{\tau} } - X^i_{t_0} \big| \mathcal{H}^i_\pi(t_0) \right] }{ \mathbb{P}\left( t_0 < \sigma^i \leq \hat{\tau}  \big| \mathcal{H}^i_\pi(t_0) \right) }\mbox{ ($\mathbb{P}$-a.e.)}.
\end{split}
\end{equation}
The last step above follows as, given that $\tau^i_\pi$ and $\tau$ are both in $\hat{\mathbb{H}}^i_\pi(t_0)$ by assumption, so too is $\tau^i_\pi \wedge \tau,$ and the term on the right hand side is the $\esssup$ over all such stopping times. 

Defining a `global' $\pi$-analog of $\rho^i$,
\begin{equation}
\rho^i_\pi(t', t'') = \frac{ \mathbb{E} \left[ X^i_{ \sigma^i \wedge t'' } - X^i_{t'} \big| \mathcal{H}^i_\pi(t') \right] }{ \mathbb{P}\left( t' < \sigma^i \leq t''  \big| \mathcal{H}^i_\pi(t') \right) },
\end{equation}
we have the following relations:
\begin{equation}
\begin{split}
\nu^i_\pi(t_0, \tau) & \leq \underset{\mbox{\ } \hat{\tau} \in \hat{\mathbb{H}}_\pi^i (t_0) }{\esssup}~\rho^i_\pi(t_0, \hat{\tau})  \leq  \underset{\mbox{\ } \hat{\tau} \in \hat{\mathbb{F}}^i (t_0) }{\esssup}~\rho^i(t_0, \hat{\tau}) \mbox{ ($\mathbb{P}$-a.e.)}.
\end{split}
\end{equation}
The first inequality above is simply a restatement of Eq. \eqref{eqn:comparison-main-bound}. The second inequality, the exchange from $\mathbb{H}^i_\pi$-stopping times to $\mathbb{F}^i$-stopping times, is intuitive: as the $X^i$ process and $\sigma^i$ are independent of the non-$i$ bandits, information about those independent bandits (through the $\mathbb{H}^i_\pi$-stopping times) cannot assist in maximizing the quotient. Rigorously, this amounts to integrating out the independent bandits; this is done in   detail as Proposition \ref{prop:exchange}. 
\qed
\end{proof} 

\begin{proposition}\label{prop:exchange}
For bandit $i$ under policy $\pi$, for any time $t_0$ such that $S^i_\pi(t_0) < \sigma_\pi$, the following holds:
\begin{equation}\label{eqn:lost-count}
\underset{\mbox{\ } \hat{\tau} \in \hat{\mathbb{H}}_\pi^i (t_0) }{\esssup}~\rho^i_\pi(t_0, \hat{\tau})  \leq \underset{\mbox{\ } \hat{\tau} \in \hat{\mathbb{F}}^i (t_0) }{\esssup}~\rho^i(t_0, \hat{\tau})  \mbox{ ($\mathbb{P}$-a.e.)}.
\end{equation}
\end{proposition}

See Section \ref{sec:exchange} for its proof. 

The following proposition provides, using $\rho^i$ and $\nu^i_\pi$,   expressions for the incremental reward gained through  consecutive or  under $\pi$ activation of a block.

\begin{proposition}\label{prop:equiv-blocks}
For each bandit $i$, the following relations hold for any $\mathbb{F}^i$-stopping times $\tau' < \tau''$ where the quantities are well defined. Equality also holds when conditioning with respect to the initial information, $\mathcal{F}^i(0)$, $\mathcal{G}_0$ respectively via the tower property.
\begin{equation}\label{eq:prop21}
\mathbb{E}^i \left[ X^i_{ \sigma^i \wedge \tau'' } - X^i_{\tau'} \big| \mathcal{F}^i(\tau') \right]   = \mathbb{E}^i \left[ \sum_{t = \tau'}^{\tau''-1} \rho^i(\tau', \tau'') \mathbbm{1}_{ \{ \sigma^i = t + 1 \} } \big| \mathcal{F}^i(\tau') \right]
\end{equation}
\begin{equation}\label{eq:prop22}
\mathbb{E} \left[ X^i_{ T^i_\pi(\sigma_\pi) \wedge \tau'' } - X^i_{\tau'} \big| \mathcal{H}^i_\pi(\tau') \right]  = \mathbb{E} \left[ \sum_{t = \tau'}^{\tau''-1} \nu^i_\pi(\tau', \tau'') \mathbbm{1}_{ \{ \sigma_\pi = S^i_\pi(t) + 1 \} } \big| \mathcal{H}^i_\pi(\tau') \right] 
\end{equation}
\end{proposition}

\begin{proof}
The above equations follow directly from Eqs. (\ref{eqn:block-value}), (\ref{eqn:pi-value}), observing the following relations:
\begin{equation}
\begin{split}
\mathbb{P}^i\left( t' < \sigma^i \leq t'' \big| \mathcal{F}^i(t') \right) & = \mathbb{E}^i \left[ \sum_{t = t'}^{t''-1} \mathbbm{1}_{ \{ \sigma^i = t + 1 \} } \big| \mathcal{F}^i(t') \right], \\
\mathbb{P}\left( t' < \sigma^i \leq \tau^i_\pi \wedge t'' \big| \mathcal{H}^i_\pi(t') \right) & = \mathbb{E} \left[ \sum_{t = t'}^{t''-1}  \mathbbm{1}_{ \{ \sigma_\pi = S^i_\pi(t) + 1 \} } \big| \mathcal{H}^i_\pi(t') \right].
\end{split}
\end{equation}
\qed
\end{proof}

\subsection{Solo Payout Indices and Times}

Theorem \ref{thm:a-value-comparison-theorem} indicates the significance of the following quantity.

\begin{definition}[The Solo-Payout Index]\label{def:res-index}
For any $t < \sigma^i$, the incremental \emph{Solo-Payout Index at $t$} is defined to be
\begin{equation}
\rho^i(t) = \underset{ \tau \in \hat{\mathbb{F}}^i (t) }{ \esssup} ~\rho^i(t, \tau).
\end{equation}
\end{definition}

This index can be interpreted as the maximal quotient of ``incremental reward'' over ``probability of termination/halting'' as in Eq. \eqref{eqn:block-value}. 
 \cc{sonin2011generalized}   defined this index for the case of finite state   Markov chain reward processes,  in order to provide an efficient computation of the Gittins indices of all states.

The following result demonstrates that $\rho^i(t)$ is realized as the value of {\sl some} block from time $t$, i.e.,   for some $\tau > t$, $\rho^i(t) = \rho^i(t, \tau) \mbox{ ($\mathbb{P}^i$-a.e.).}$ As such, $\rho^i(t)$ represents the {\sl maximal block value achievable from process $i$ from time $t$}.

\begin{proposition}\label{prop2}
For any time $t_0 < \sigma^i$, there exists a $\tau \in  \hat{ \mathbb{F}}^i  (t_0)$ such that $\rho^i(t_0) = \rho^i(t_0, \tau) \mbox{ \emph{($\mathbb{P}^i$-a.e.)}}$.
\end{proposition}

The proof  is relegated to Section \ref{sec:exchange}, since it specialized  and not the focus of this paper.

The solo-payout indices and their realizing blocks provide a natural time scale with which to view a process, in terms of a sequence of blocks. In particular, we define the following sequence:

\begin{definition}[Solo-Payout Index Times]\label{def:ris-sequence}
Define a sequence of $\mathbb{F}^i$-stopping times $\{ \tau^i_k \}_{k \geq 0}$ in the following way, that $\tau^i_0 = 0$, and for $k > 0$, 
\begin{equation}\label{eqn:gittins-sequence-definitions}
\tau^i_{k+1} =  \text{\emph{arg}}\esssup \{ \rho^i(\tau^i_k, \tau) : \tau \in \hat{\mathbb{F}}^i(\tau^i_k) \}.
\end{equation}  
\end{definition}

In the case that $\tau^i_k = \sigma^i$ for some $k$, then $\tau^i_{k'}$ is taken to be infinite for all  $k'>k$. In the case that $\tau^i_k < \sigma^i$, we have that $\rho^i(\tau^i_k) = \rho^i(\tau^i_k, \tau^i_{k+1})$. The question of whether the `arg$\esssup$' exists is resolved in the positive by Proposition \ref{prop2}; if there is more than one stopping time that attains the `arg$\esssup$', we take $\tau^i_{k+1}$ to be the one demonstrated by the application of Lemma \ref{lem:optional} in the proof of Proposition \ref{prop2}.

Using this sequence of stopping times, we partition the local process times $\mathbb{N}^i=\{0,1,2,\ldots\}$ into 
$$\mathbb{N}^i=[0,\tau^i_1)\cup[\tau^i_1,\tau^i_2)\cup[\tau^i_2,\tau^i_3)\cup\ldots .$$
One important property of this partition is the following:

\begin{proposition}[Solo-Payout Indices are Non-Increasing over Index Times]\label{prop:gittins-decreasing}

For any $k > 0$ such that $\tau^i_k < \sigma^i$, the following is true: $\rho^i(\tau^i_{k-1}) \geq \rho^i(\tau^i_k) \mbox{  \emph{($\mathbb{P}^i$-a.e.).}}$
\end{proposition}

For intuition, recall the $\{ \tau^i_k \}_k$ are meant to realize successively the maximal indices of the process $\{ X^i_t \}_t$. If $\rho^i(\tau^i_{k-1}) = \rho^i(\tau^i_{k-1}, \tau^i_k) < \rho^i(\tau^i_k)$, the index from $\tau^i_{k-1}$ may be increased by taking a block that extends from $\tau^i_{k-1}$ \textit{past} $\tau^i_k$. This contradicts the idea of the $\{ \tau^i_k \}_k$ as realizing the maximal indices. The proof is relegated to   Section \ref{sec:exchange}, as technical, and not the focus of this paper.

\subsection{Equivalent Solo Payout Processes}

For each bandit, we have developed a partition of local time into blocks of activations via the solo payout index stopping times. With Proposition \ref{prop:equiv-blocks} in mind, we use these blocks to define a set of reward equivalent penultimate solo payout processes, and $\pi$-equivalent solo payout processes.

\begin{definition}\label{def:equiv}
Given the collection of \emph{reward processes} $\mathbb{X} = (X^1, ...,X^N)$, and ${ \{ \tau^i_k \} }_{k \geq 0}$ for each $i$ as in Definition \ref{def:ris-sequence}, we define:
\begin{enumerate}
\item The \emph{reward-equivalent solo payout collection} $\mathbb{Y}^X = (Y^1, ...,Y^N)$ by
\begin{equation}\label{eq:xrho}
Y^i (t) = \rho^i(\tau^i_k), \ \ \mbox{ if } \tau^i_k \leq t < \tau^i_{k+1}.
\end{equation}

\item For $\pi\in\mathcal{P}$, the \emph{$\pi$-equivalent solo payout  collection} $\mathbb{Y}^X_\pi = (Y_\pi^1, ..., Y_\pi^N)$, by
\begin{equation}\label{eq:xhat}
Y_\pi^i(t) =  \nu^i_\pi(\tau^i_k, \tau^i_{k+1}), \ \ \mbox{ if } \tau^i_k \leq t < \tau^i_{k+1} .
\end{equation}
\end{enumerate}
\end{definition}

Like $X^i$, the process $Y^i$ is defined on $(\Omega^i, \mathcal{F}^i, \mathbb{P}^i, \mathbb{F}^i)$ and is $\mathbb{F}^i$-adapted, as the $\rho^i(\tau^i_k)$ is defined by the information available locally at time $\tau^i_k$. However, as the $\nu^i_\pi(\tau^i_k, \tau^i_{k+1})$ depend on the specifics of policy $\pi$, so do the $Y^i_\pi$ processes; the $Y^i_\pi$ processes are $\mathbb{H}_\pi^i$-adapted, but not $\mathbb{F}^i$-adapted. Note, $Y^i$ is only really defined for $t < \sigma^i$, and $Y^i_\pi$ is only defined for $t$ such that $S^i_\pi(t) < \sigma_\pi$. However, since no rewards are collected from bandit $i$ after these times, this lack of definition is of no consequence. 

The following are simple, but important properties of the $\mathbb{Y}^X, \mathbb{Y}^X_\pi$ processes.

\begin{proposition}\label{prop:equiv-blocks2}
For $\pi \in \mathcal{P}$, for each $i$, and any $k$ where the following quantities are defined,
\begin{equation}\label{eq:prop21b}
\mathbb{E}^i \left[ X^i_{ \sigma^i \wedge \tau^i_{k+1}  } - X^i_{\tau^i_k} \big| \mathcal{F}^i(\tau^i_k) \right]   = \mathbb{E}^i \left[ \sum_{t = \tau^i_k}^{\tau^i_{k+1}-1} Y^i(t) \mathbbm{1}_{ \{ \sigma^i = t + 1 \} } \big| \mathcal{F}^i(\tau^i_k) \right],
\end{equation}
\begin{equation}\label{eq:prop22b}
\mathbb{E} \left[ X^i_{ T^i_\pi(\sigma_\pi)  \wedge \tau^i_{k+1} } - X^i_{\tau^i_k} \big| \mathcal{H}^i_\pi(\tau^i_k) \right]  = \mathbb{E} \left[ \sum_{t = \tau^i_k}^{\tau^i_{k+1}-1} Y_\pi^i(t) \mathbbm{1}_{ \{ \sigma_\pi = S^i_\pi(t) + 1 \} } \big| \mathcal{H}^i_\pi(\tau^i_k) \right].
\end{equation}
As with Proposition \ref{prop:equiv-blocks}, equality also holds when conditioning with respect to $\mathcal{F}^i(0), \mathcal{G}_0$.

\begin{proof}
This follows as an application of Proposition \ref{prop:equiv-blocks} and the definitions of $Y^i$, $Y^i_\pi$.
\end{proof}
\end{proposition}

The following proposition serves as justification of the term ``equivalent'' in describing the $\mathbb{Y}^X, \mathbb{Y}^X_\pi$ collections.

\begin{proposition}\label{equiv-bandits}
For each $i$, for any policy $\pi \in \mathcal{P}$,
\begin{equation}
\mathbb{E}^i \left[ X^i_{\sigma^i} \big| \mathcal{F}^i(0) \right]   = \mathbb{E}^i \left[ Y^i(\sigma^i-1) \big| \mathcal{F}^i(0) \right],
\end{equation}
\begin{equation}\label{equiv-bandits:policy}
\mathbb{E} \left[ X^i_{T^i_\pi(\sigma_\pi)} \big| \mathcal{G}_0  \right]  = \mathbb{E} \left[ \mathbbm{1}_{ \{ i = \pi(\sigma_\pi-1) \} } Y^i_\pi(T^i_\pi(\sigma_\pi-1)) \big| \mathcal{G}_0 \right].
\end{equation}
\begin{proof}
Each follows from the corresponding equation in Proposition  \ref{prop:equiv-blocks2}, summing over $k$ and taking expectations from the initial time, via the tower property. On the right hand sides, the $X^i$ terms telescope in the sum, and $X^i_0$ is taken to be 0. On the left hand sides, the sums over $Y$ may be expressed as single terms, due to the indicators.
\end{proof}
\end{proposition}

\begin{proposition}\label{non-increasing}
For each $i$, and any   time $t > 0$ such that $Y^i(t)$ is well defined,
\begin{equation}
Y^i(t-1) \geq Y^i(t) \mbox{ ($\mathbb{P}^i$-a.e.) }.
\end{equation}
\begin{proof}
This follows immediately from Proposition \ref{prop:gittins-decreasing}, and Definition \ref{def:equiv}.1.
\end{proof}
\end{proposition}

\begin{theorem}[Comparison of Equivalent, $\pi$-Equivalent Solo Payout Processes]\label{cor:2}
For any $\pi \in \mathcal{P}$, for each $i$ and all time $t$ where both are defined, we have:
\begin{equation}
Y_\pi^i(t) \leq Y^i(t)\ \ \mbox{ \ \  \emph{($\mathbb{P}$-a.e.)}}.
\end{equation}
 
\begin{proof}
For such a $t$, we have for some $k$ that $\tau^i_k \leq t < \tau^i_{k+1}$, and as an application of Theorem 1,
\begin{equation}
Y_\pi^i(t) =  \nu^i_\pi(\tau^i_k, \tau^i_{k+1}) \leq 
\underset{ \mbox{\ }\tau' \in \hat{\mathbb{H}}^i_\pi (\tau^i_k) }{\esssup}~\nu^i_\pi(\tau^i_k, \tau') \leq \underset{\mbox{\ } \hat{\tau} \in \hat{\mathbb{F}}^i (\tau^i_k) }{\esssup}~\rho^i(\tau^i_k, \hat{\tau})  =   \rho^i (\tau^i_k ) = Y^i(t) \  \ \mbox{ ($\mathbb{P}$-a.e.)}.
\end{equation}
Note in the above that the first and the last  relations are  just definitions,
 the second follows naturally by comparing one instance of the function to an 
 $\esssup$ of the same function, the third is due to Theorem \ref{thm:a-value-comparison-theorem}, the fourth is due to the definition of the $ \rho^i $ function. \qed
 
\end{proof}
\end{theorem}

\subsection{The Optimal Policy}

The derivation of the optimal control policy for an arbitrary collection of reward processes $\mathbb{X}$ under a collective  reward structure is all but immediate now.

\begin{theorem}[The Optimal Collective Payout Control Policy]\label{the:optimal}
For a collection of reward processes $\mathbb{X} = (X^1, X^2, \ldots, X^N)$, and the associated stopping times ${ \{ \sigma^i \} }_{i = 1, \ldots, N}$, there exists a strategy $\pi^* \in \mathcal{P}$ such that for all $\pi \in \mathcal{P}$,
\begin{equation}
V^{CP}_\pi(\mathbb{X}) \leq V^{CP}_{\pi^*}(\mathbb{X}) \mbox{ ($\mathbb{P}$-a.e.)}.
\end{equation}
In particular, such an optimal policy $\pi^*$ can be described in the following way: successively activate the bandit with the largest current solo payout index,
\begin{equation}\label{eq:roi}
\rho^i(t) = \underset{ \tau \in \hat{\mathbb{F}}^i (t) }{ \esssup} ~ \frac{ \mathbb{E}^i \left[ X^i_{ \sigma^i \wedge \tau } - X^i_{t} \big| \mathcal{F}^i(t) \right] }{ \mathbb{P}^i\left( t < \sigma^i \leq \tau \big| \mathcal{F}^i(t) \right) },
\end{equation}
for the duration of the corresponding index block.
\end{theorem}

Before giving the proof of this theorem, we give a corollary, which gives a useful alternative characterization of the policy $\pi^*$.
\begin{corollary}
An alternative characterization of the policy $\pi^*$ in Theorem \ref{the:optimal} is the following: at every round, activate the bandit with the largest current solo payout index.
\end{corollary}
\begin{proof}
From Theorem \ref{the:optimal}, it follows that the optimal first activation is to activate a bandit with the largest current solo payout index. If that activation does not halt the bandit and end the game, the controller is faced with a structurally identical decision problem. It follows that again, the optimal activation is to activate a bandit with the largest current solo payout index. This argument may be iterated until halting, which will occur in finite time by assumption on the $\{ \sigma^i \}$.
\end{proof}

{\bf Proof of Theorem \ref{the:optimal}.}
For an arbitrary policy $\pi$, and $\pi^*$ as indicated above, we establish the following relations:
\begin{equation}\label{eq:finalstep}
V^{CP}_\pi( \mathbb{X} ) = V^{PSP}_\pi( \mathbb{Y}^X_\pi ) \leq V^{PSP}_\pi( \mathbb{Y}^X) \leq V^{PSP}_{\pi^*}( \mathbb{Y}^X) = V^{CP}_{\pi^*}(\mathbb{X}) \mbox{ ($\mathbb{P}$-a.e.)},
\end{equation}
%

i.e., for any policy $\pi$, we have that
$
V^{CP}_\pi( \mathbb{X} ) \leq V^{CP}_{\pi^*}(\mathbb{X}) \mbox{ ($\mathbb{P}$-a.e.) }
$ and 
therefore $\pi^*$ is an optimal policy.

In the following steps we prove  relations  (\ref{eq:finalstep}).

\emph{Step 1:} $V^{CP}_\pi( \mathbb{X}) = V^{PSP}_\pi( \mathbb{Y}^X_\pi )$, \pae. 


We have, via Prop. \ref{equiv-bandits}, Eq. \eqref{equiv-bandits:policy},
\begin{equation*}
V^{CP}_\pi(\mathbb{X}) = \sum_{i = 1}^N \mathbb{E} \left[ X^i_{T^i_\pi(\sigma_\pi)} \big| \mathcal{G}_0 \right] = \sum_{i = 1}^N \mathbb{E} \left[ \mathbbm{1}_{ \{ i = \pi(\sigma_\pi-1) \} } Y^i_\pi(T^i_\pi(\sigma_\pi-1)) \big| \mathcal{G}_0 \right] = V^{PSP}_\pi( \mathbb{Y}^X_\pi ).
\end{equation*}

Note, because the $Y^i_\pi$ processes are defined in terms of $\pi$, they are not $\mathbb{F}^i$-adapted, and cannot be utilized under any other policy. However, the value $V^{PSP}_\pi( \mathbb{Y}^X_\pi )$ is well defined via the above equation.

\emph{Step 2:} 
$V^{PSP}_\pi( \mathbb{Y}^X_\pi ) \leq V^{PSP}_\pi( \mathbb{Y}^X) \mbox{ ($\mathbb{P}$-a.e.)}$.

This follows from the point-wise inequality of Theorem \ref{cor:2}, $Y^i_\pi(t) \leq Y^i(t) $ for all $t$. Note that for any $t$ where $Y^i_\pi(t)$ is not defined, the $t^{th}$ activation of $i$ does not occur under $\pi$, and no comparison is necessary.

\emph{Step 3:} 
 $V^{PSP}_\pi( \mathbb{Y}^X) \leq V^{PSP}_{\pi^*}( \mathbb{Y}^X) \mbox{ ($\mathbb{P}$-a.e.)}$. 
 
This follows simply from Theorem \ref{thm:greedy} as, by construction, the terms of each $Y^i$ process are equal to the solo payout indices of $X^i$, piecewise constant over blocks, and non-increasing.

\emph{Step 4:} $V^{PSP}_{\pi^*}( \mathbb{Y}^X) = V^{CP}_{\pi^*}(\mathbb{X}) \mbox{ ($\mathbb{P}$-a.e.). }$

Note that $\pi^*$ activates bandits consecutively over the duration of their index blocks. For a given $i$, define
\begin{equation}
k^*_i = \min_{k \geq 0} \{ S^i_{\pi^*}(\tau^i_k) \geq \sigma_\pi \},
\end{equation}
the first block of $i$ that is {\sl not} activated under $\pi^*$. Note then that for each $i$, we have the following relation
\begin{equation}
T^i_{\pi^*}(\sigma_{\pi^*}) = \sigma^i \wedge \tau^i_{k^*_i}.
\end{equation}

Expressing the value of policy $\pi^*$ relative to activations over blocks, and utilizing the tower property, we have the following equivalences:
\begin{equation}
\begin{split}
V^{PSP}_{\pi^*}( \mathbb{Y}^X ) & = \sum_{i = 1}^N \sum_{k = 0}^{\infty} \mathbb{E} \left[  \mathbbm{1}_{\{ k^*_i > k \}} \sum_{t = \tau^i_{k}}^{\tau^i_{k+1}-1} Y^i(t) \mathbbm{1}_{ \{ \sigma^i = t + 1 \} } \big| \mathcal{G}_0 \right] \\
& = \sum_{i = 1}^N \sum_{k = 0}^{\infty}  \mathbb{E} \left[ \mathbbm{1}_{\{ k^*_i > k \}} \mathbb{E} \left[ \sum_{t = \tau^i_{k}}^{\tau^i_{k+1}-1} Y^i(t) \mathbbm{1}_{ \{ \sigma^i = t + 1 \} } \big| \mathcal{H}^i_\pi(\tau^i_k) \right] \big| \mathcal{G}_0 \right] \\
& = \sum_{i = 1}^N \sum_{k = 0}^{\infty} \mathbb{E} \left[  \mathbbm{1}_{\{ k^*_i > k \}} \mathbb{E} \left[ X^i_{ \sigma^i \wedge \tau^i_{k+1}  } - X^i_{\tau^i_k} \big| \mathcal{H}^i_\pi(\tau^i_k) \right] \big| \mathcal{G}_0 \right] \\
& = \sum_{i = 1}^N \mathbb{E} \left[ X^i_{ \sigma^i \wedge \tau^i_{k^*_i}  } - X^i_{0}  \big| \mathcal{G}_0 \right]\\
&  = \sum_{i = 1}^N \mathbb{E} \left[ X^i_{ T^i_{\pi^*}(\sigma_{\pi^*})  } \big| \mathcal{G}_0 \right]  = V^{CP}_{\pi^*} ( \mathbb{X} ).
\end{split}
\end{equation}

Note the exchange over blocks of the $Y^i$ rewards for the $X^i$ rewards is due to Proposition \ref{prop:equiv-blocks2}, Eq. \eqref{eq:prop21b}, taking the extension to $\mathcal{H}^i_{\pi^*}(\tau^i_k)$ in place of $\mathcal{F}^i(\tau^i_k)$.
\qed

\begin{remark} The above theorem demonstrates a policy $\pi^* \in \mathcal{P}$ that is $\mathbb{P}$-a.e. superior (or equivalent) to every other policy $\pi \in \mathcal{P}$. However, the set of non-anticipatory policies $\mathcal{P}$ was defined in a fairly restrictive sense in Sec. \ref{sec:information}, so that the decision in any round was completely determined by the results of the past. This might be weakened to allow for randomized policies, so that the decision in a given round might depend on the results of independent events, e.g., coin flips. However, such a construction simply amounts to placing a distribution on $\mathcal{P}$. Since $\pi^*$ is $\mathbb{P}$-a.e. superior to any $\pi \in \mathcal{P}$, $\pi^*$ would be similarly superior to any policy sampled randomly from $\mathcal{P}$.
\end{remark}

The structure of the proof of Theorem \ref{the:optimal} above is based on deriving an optimality result for the collective payout model  by reducing it to an instance of the a   solo payout model. It  suggests an interesting correspondence between the two. Under the collective payout model, in any period the controller wishes to achieve via bandit activation high collective rewards   of all bandits on halting. Under a solo payout model, in any period the controller wishes to achieve via bandit activation high rewards  of a given bandit on halting. However, since (under either model) the controller can only activate one bandit at a time, under the collective payout model the controller essentially seeks in every period to maximize the change in collective reward due to a single bandit \textit{should that bandit halt}, or equivalently to maximize the change in reward \textit{of that single bandit} should that bandit halt. The collective payout model can therefore be cast as a penultimate solo payout model, where the payout on halting is based on the \textit{change} in reward of the activated bandit rather than the final collective rewards of all bandits.  This can be further seen in the following section, where   optimal index policies  for the general (penultimate and ultimate) solo payout model are given.

\section{Additional  Payout Schemes}\label{sec:alternative}

Utilizing the results of the previous section, we next provide index policies for optimizing the rewards/costs from a number of additional payout models, by reducing them to the collective payout model cf. Eq.(\ref{eq:CP}) of the previous section, and utilizing Theorem \ref{the:optimal}. We construct  the models below, specified by different ways in which rewards are received and/or costs are paid. We note that 
analogous results can be obtained for the penultimate solo payout model without  the 
monotonicity restriction of Section 3 on the underlying reward processes. They are omitted for brevity. 
 
\pbsaa The \emph{Ultimate Solo Payout} model (SP). In this model the controller aims to  
maximize the expected final reward from the bandit that halts the game, i.e., 
  the value of  a policy $\pi$ is defined as,
\begin{equation}
V^{SP}_\pi(\mathbb{X})  = \mathbb{E} \left[ X_\pi(\sigma_\pi) | \mathcal{G}_0 \right] = \sum_{i = 1}^N \mathbb{E} \left[ \mathbbm{1}_{ \{ i = \pi(\sigma_\pi-1) \} } X^i_{T^i_\pi(\sigma_\pi)} | \mathcal{G}_0\right].
\end{equation}

\pbsaa 
The \emph{Non-Halting Cost} model (NH).
In this model the controller \emph{pays a cost} based on the bandits that did not halt the game, and wishes to minimize this expected cost. The \emph{halting cost} of a policy $\pi$ is
\begin{equation}
V^{NH}_\pi(\mathbb{X})  = \mathbb{E} \left[ \sum_{i \neq \pi(\sigma_\pi-1)} X^i_{T^i_\pi(\sigma_\pi)} | \mathcal{G}_0 \right]\\
 = \sum_{i = 1}^N \mathbb{E} \left[ \mathbbm{1}_{ \{ i \neq \pi(\sigma_\pi-1) \} } X^i_{T^i_\pi(\sigma_\pi)} | \mathcal{G}_0\right].
\end{equation}

\pbsaa
The \emph{Total Profit} model (TP). In this model  to each bandit $i$ we associate a reward process $\{ R^i_t \}_{t \geq 0}$ and a cost process $\{ C^i_t \}_{t \geq 0}$. The controller gains a reward from the bandit that halts the game, and pays a cost for each bandit that does not halt. The controller wishes to maximize her expected total profit, i.e., 
  the value of  a policy $\pi$ is now defined as,  
\begin{equation}
\begin{split}
V^{TP}_\pi(\mathbb{R}, \mathbb{C}) & = \sum_{i = 1}^N \mathbb{E} \left[ \mathbbm{1}_{ \{ i = \pi(\sigma_\pi-1) \} } R^i_{T^i_\pi(\sigma_\pi)} - \mathbbm{1}_{ \{ i \neq \pi(\sigma_\pi-1) \} } C^i_{T^i_\pi(\sigma_\pi)}  | \mathcal{G}_0\right].
\end{split}
\end{equation}

\pbsaa
The \emph{Cumulative Collective Payout} model (CCP) and the Gittins Index.
In this model  the controller gains a bandit's current reward each time that bandit is chosen to be activated. Bandits that are never activated give no rewards. The controller wishes to maximize her expected total payout, i.e., 
  the value of  a policy $\pi$  is now defined as,  
\begin{equation}
V^{CCP}_\pi(\mathbb{X}) =\sum_{i = 1}^N \mathbb{E}\left[ \sum_{t = 0}^{T^i_\pi( \sigma_\pi ) - 1} X^i_t \big| \mathcal{G}_0 \right].
\end{equation}

Note, in the above expressions we take empty sums to be 0. 

For all these models, we will provide an index policy to maximize the corresponding value function as follows. 

\rbsaa For the {\sl SP model}, define a collection of reward processes $\mathbb{Z} = \{ Z^i \}_{1 \leq i \leq N}$ by for each $i$, each $t \geq 0$,
\begin{equation}
Z^i_t = \mathbbm{1}_{ \{ \sigma^i = t \} } X^i_t.
\end{equation}
Notice that at round $\sigma_\pi$, $Z^i_t = 0$ for all bandits that did not halt the game, and $Z^i_t = X^i_{\sigma^i}$ for the bandit that did halt the game. Hence the collective payout under $\mathbb{Z}$ is equal to the solo payout under $\mathbb{X}$, $V^{CP}_\pi(\mathbb{Z}) = V^{SP}_\pi(\mathbb{X})$. Applying Theorem \ref{the:optimal}, the optimal policy for the collective payout under $\mathbb{Z}$ yields an optimal policy for the solo payout under $\mathbb{X}$, and it is  given by a policy that always activates bandits according to the maximum \textit{solo payout index}:
\begin{equation}
\rho^i_{SP}(t) = \underset{ \tau \in \hat{\mathbb{F}}^i (t) }{ \esssup} ~\frac{ \mathbb{E}^i \left[ \mathbbm{1}_{ \{ \tau \geq \sigma^i \} } X^i_{ \sigma^i } \big| \mathcal{F}^i(t) \right] }{ \mathbb{P}^i\left( t < \sigma^i \leq \tau \big| \mathcal{F}^i(t) \right) }.
\end{equation}
It is interesting to observe that the policy based on the above index has a very natural interpretation, viewing the index as the maximal conditional expected payout of a bandit on its halting, i.e., the policy always activates the bandit with the largest potential payout - should it pay out.  Additionally, comparing the above index to the optimal index for the collective payout model, it is clear that the collective payout index emphasizes the ``change in reward on halting'' of a single bandit, while the solo payout index emphasizes only the final reward of a single bandit on halting. This again highlights the correspondence between these two models, as discussed at the end of Section 4.4.

\pbsaa
We   reduce the {\sl  Non-Halting Cost}  model,    to the collective payout  model in the following way. Define a collection of reward processes $\mathbb{Z} = \{ Z^i \}_{1 \leq i \leq N}$ by for each $i$, each $t \geq 0$,
\begin{equation}\label{eq:nhc}
Z^i_t = -\mathbbm{1}_{ \{ \sigma^i \neq t \} } X^i_t.
\end{equation}
Notice that at round $\sigma_\pi$, if bandit $i$ was activated to halt the game (i.e., $\pi(\sigma_\pi-1) = i$), Eq.(\ref{eq:nhc}) implies that  $Z^i_t = 0$ and $Z^j_t =  -X^j_t$, for $j\neq i$. Hence, the collective payout under $\mathbb{Z}$ is equal to the negative of the halting cost under $\mathbb{X}$: $V^{CP}_\pi(\mathbb{Z}) = - V^{NH}_\pi(\mathbb{X})$; it follows that  maximizing the collective payout under $\mathbb{Z}$  minimizes the halting cost under $\mathbb{X}$. 
Applying Theorem \ref{the:optimal}, the optimal policy for the collective payout under $\mathbb{Z}$ yields an optimal policy for the non-halting cost model under $\mathbb{X}$, and it is  given by a policy   that always activates bandits according to the minimum \textit{non-halting cost index}:
\begin{equation}
\rho^i_{NH}(t) =  \underset{ \tau \in \hat{\mathbb{F}}^i (t) }{ \esssup} ~ \frac{ \mathbb{E}^i \left[ \mathbbm{1}_{ \{ \sigma^i > \tau \} } X^i_{\tau} - X^i_t   \big| \mathcal{F}^i(t) \right] }{ \mathbb{P}^i\left( t < \sigma^i \leq \tau \big| \mathcal{F}^i(t) \right) }.
\end{equation}

 \pbsaa
For the \emph{Total Profit} model (TP) model, in order to  
  provide an index policy to maximize its value function, we reduce it to the collective payout  model in the following way. Define a collection of reward processes $\mathbb{Z} = \{ Z^i \}_{1 \leq i \leq N}$ by for each $i$, each $t \geq 0$,
\begin{equation}
Z^i_t = \mathbbm{1}_{ \{ \sigma^i = t \} } R^i_t - \mathbbm{1}_{ \{ \sigma^i \neq t \} } C^i_t.
\end{equation}
Notice that at round $\sigma_\pi$, $Z^i_t = -C^i_t$ for all bandits that did not halt the game, and $Z^i_t = R^i_{t}$ for the bandit that did halt the game. Hence the collective payout under $\mathbb{Z}$ is equal to the collective profit solo payout under $(\mathbb{R}, \mathbb{C})$, $V^{CP}_\pi(\mathbb{Z}) = V^{TP}_\pi(\mathbb{R}, \mathbb{C})$, cf. Eq.(\ref{eq:CP}). Thus, as before, 
 the optimal policy for the collective payout under $\mathbb{Z}$ yields an optimal policy for the total profit under $(\mathbb{R}, \mathbb{C})$, given by a policy that always activates bandits according to the maximum \textit{total profit index}:
\begin{equation}
\rho^i_{TP}(t) = \underset{ \tau \in \hat{\mathbb{F}}^i (t) }{ \esssup} ~ \frac{ \mathbb{E}^i \left[  \mathbbm{1}_{ \{ \sigma^i \leq \tau \} } R^i_{\sigma^i} - \mathbbm{1}_{ \{ \sigma^i > \tau \} } C^i_{\tau}+ C^i_t  \big| \mathcal{F}^i(t) \right] }{ \mathbb{P}^i\left( t < \sigma^i \leq \tau \big| \mathcal{F}^i(t) \right) }.
\end{equation}

 \pbsaa
For the \emph{Cumulative Collective Payout} model (CCP) .
In this model  the controller gains a bandit's current reward each time that bandit is chosen to be activated. Bandits that are never activated give no rewards.  
To provide an index policy to maximize this value function, we reduce it to   the collective payout  
 model, in the following way.  Define a collection of reward processes $\mathbb{Z} = \{ Z^i \}_{1 \leq i \leq N}$ by
\begin{equation}
Z^i_t = \sum_{t' = 0}^{t-1} X^i_{t'} \mbox{, \  for each $i$, each $t \geq 0$}.
\end{equation}
It follows easily that the collective payout model value under $\mathbb{Z}$ is equal to the collective cumulative payout under $\mathbb{X}$, i.e.,  $V^{CP}_\pi(\mathbb{Z}) = V^{CCP}_\pi(\mathbb{X})$. 
Thus, 
applying Theorem \ref{the:optimal}, the optimal policy for the collective payout under $\mathbb{Z}$    yields an optimal policy for the collective cumulative payout under $\mathbb{X}$, given by a policy that always activates bandits according to the maximum \textit{collective cumulative payout index}:
\begin{equation}
\rho^i_{CCP}(t) = \underset{ \tau \in \hat{\mathbb{F}}^i (t) }{ \esssup} ~ \frac{ \mathbb{E}^i \left[ \sum_{t' = t}^{\sigma^i \wedge \tau-1} X^i_{t'} \big| \mathcal{F}^i(t) \right] }{ \mathbb{P}^i\left( t < \sigma^i \leq \tau \big| \mathcal{F}^i(t) \right) }.
\end{equation}
This extension of the collective payout model  is interesting in its own right, because  it allows us to readily recover and provide new simple proofs for the classic result of \cc{Gittins:79} and the recent results in  \cc{cowan2015multi}. 

Indeed,   consider the case in which each time the controller activates a bandit, all   future expected rewards are  effectively  discounted by a factor equal to the probability of that decision not halting the game. 

In the special case that  
each halting time $\sigma^i > 0$ is a geometric random variable with a constant parameter $0 < \beta < 1$, independent of the reward processes $\mathbb{X}$, 
i.e., $\mathbb{P}^i( \sigma^i = t + 1 | \mathcal{F}^i(t) )=1-\beta$. This results in every activation discounting all future rewards by a factor of $\beta$.

 It is easy to see that
\begin{equation}\label{eq:ccp}
V^{CCP}_\pi(\mathbb{X}) = \sum_{i = 1}^N \mathbb{E}\left[ \sum_{t = 0}^{T^i_\pi( \sigma_\pi ) - 1} X^i_t \big| \mathcal{G}_0 \right] = \mathbb{E}\left[ \sum_{s = 0}^\infty \beta^s X_\pi(s) \big| \mathcal{G}_0 \right].
\end{equation}
It follows from Eq. (\ref{eq:ccp}), that maximizing the $V^{CCP}_\pi(\mathbb{X})$ under this model (with  $\mathbb{P}^i( \sigma^i = t + 1 | \mathcal{F}^i(t) )=1-\beta$, for all $t$ and all $i$) is then equivalent 
precisely the framework outlined by \cc{Gittins:79}, 
 i.e.,  total expected discounted reward of $\mathbb{X}$ for a  constant discount factor $\beta$.  
In this case, the collective cumulative payout index reduces to
\begin{equation}
\rho^i_{CCP}(t) = \underset{ \tau \in \hat{\mathbb{F}}^i (t) }{ \esssup} ~\frac{ \mathbb{E}^i \left[ \sum_{t' = t}^{\tau-1} \beta^{t'-t} X^i_{t'} \big| \mathcal{F}^i(t) \right] }{ \mathbb{E}^i\left[1 - \beta^{\tau-t}\big| \mathcal{F}^i(t) \right] } =\frac{1}{1-\beta} \underset{ \tau \in \hat{\mathbb{F}}^i (t) }{ \esssup} ~\frac{ \mathbb{E}^i \left[ \sum_{t' = t}^{\tau-1} \beta^{t'} X^i_{t'} \big| \mathcal{F}^i(t) \right] }{ \mathbb{E}^i\left[ \sum_{t' = t}^{\tau-1} \beta^{t'} \big| \mathcal{F}^i(t) \right] },
\end{equation}
where the {\sl essential sup} on the right is precisely the Gittins index for bandit $i$. As $1/(1-\beta)$ is a constant, positive factor, activating according to the maximal collective cumulative payout index and activating according to the maximal Gittins index result in equivalent, optimal policies. We also note that in this $\rho^i_{CCP}(t)$ is   the restart index cf.  
\cc{Kat87} and the generalized index of \cc{sonin2008generalized}.

 The above  is a well known interpretation  of the Gittins index problem in terms of halting bandits, but  its  treatment herein provides  a new interesting implication. In its classical form, it is not intuitively clear why the decision problem decomposes into indices that treat each bandit separately. However, framing it as a collective payout halting problem, we may make use of the previously described correspondence with the solo payout model. Reducing the Gittins model to a solo payout model, where  in every  period the controller wishes to realize the largest change in value of a single bandit on halting, provides additional insight into why  the decomposition of the decision process into treating each bandit independently holds.  

 We additionally note that the above arguments can be extended  to generalized 
 sequences of discount factors, for which  $\mathbb{P}^i( \sigma^i = t + 1 | \mathcal{F}^i(t) )=1-\beta^i_t$, and thus recover the main results of 
 \cc{cowan2015multi}.

\section{Proofs of auxiliary propositions}\label{sec:exchange}

We start with the following. 
 
 {\bf Proof of Proposition \ref{prop:exchange}}.
Without loss of generality, we may take $t_0 = 0$. Recall the definition of $\rho^i_\pi, \rho^i$:
\begin{equation}
\begin{split}
\rho^i_\pi(t', t'') & = \frac{ \mathbb{E} \left[ X^i_{ \sigma^i \wedge t'' } - X^i_{t'} \big| \mathcal{H}^i_\pi(t') \right] }{ \mathbb{P}\left( t' < \sigma^i \leq t''  \big| \mathcal{H}^i_\pi(t') \right) },\\
\rho^i(t', t'') & = \frac{ \mathbb{E}^i \left[ X^i_{ \sigma^i \wedge t'' } - X^i_{t'} \big| \mathcal{F}^i(t') \right] }{ \mathbb{P}^i\left( t' < \sigma^i \leq t'' \big| \mathcal{F}^i(t') \right) }.
\end{split}
\end{equation}
Letting $R$ denote the R.H.S. of Eq. \eqref{eqn:lost-count}, observe (by the definition of the $\esssup$) that for any $\hat{\tau} \in \hat{\mathbb{F}}^i(0)$,
\begin{equation}\label{eqn:f-eqn}
 \mathbb{E} \left[ X^i_{ \sigma^i \wedge \hat{\tau} } - X^i_{0} - R \mathbbm{1}\{ 0 < \sigma^i \leq \hat{\tau} \} \big| \mathcal{F}^i(0) \right] \leq 0 \mbox{ ($\mathbb{P}$-a.e.)}.
\end{equation}
To prove the proposition, it suffices to show that for any $\hat{\tau} \in  \hat{\mathbb{H}}_\pi^i (0)$,
\begin{equation}
 \mathbb{E} \left[ X^i_{ \sigma^i \wedge \hat{\tau} } - X^i_{0} - R \mathbbm{1}\{ 0 < \sigma^i \leq \hat{\tau} \} \big| \mathcal{H}^i_\pi(0) \right] \leq 0 \mbox{ ($\mathbb{P}$-a.e.)}.
\end{equation}
For compactness of argument, we take $N = 2$ and $i = 1$, though the following argument generalizes to arbitrary bandits in the obvious way. For notational compactness, we define $W^i_t = X^i_{ \sigma^i \wedge t } - X^i_{0} - R \mathbbm{1}\{ 0 < \sigma^i \leq t \}.$

Note that for any set $A \in \mathcal{H}^1_\pi(0)$, and any $\tau \in \hat{\mathbb{H}}^1_\pi(0)$,
\begin{equation}
\mathbb{E}\left[ \mathbbm{1}_A \mathbb{E}\left[ W^1_\tau \big| \mathcal{H}^1_\pi(0) \right]\right] = \mathbb{E}\left[ \mathbbm{1}_A W^1_\tau \right].
\end{equation}
Taking $A$ as a rectangle in $\mathcal{H}^1_\pi(0)$, $A = A_1 \times A_2$, observe that $A_1 \in \mathcal{F}^1(0)$. The indicator may be decomposed as $\mathbbm{1}_A(\omega) = \mathbbm{1}_{A_1}(\omega^1)\mathbbm{1}_{A_2}(\omega^2)$. It follows as a result of the initial integrability assumptions on the bandits, Eqs. \eqref{eqn:integrability}, \eqref{eqn:transitions}, that we may exchange the expectation over the product space for an iterated expectation:
\begin{equation}
\begin{split}
\mathbb{E}\left[ \mathbbm{1}_A W^1_\tau \right] & = \mathbb{E}^2\left[ \mathbb{E}^1\left[ \mathbbm{1}_{A_1}\mathbbm{1}_{A_2} W^1_\tau \right] \right]\\
& = \mathbb{E}^2\left[ \mathbbm{1}_{A_2} \mathbb{E}^1\left[ \mathbbm{1}_{A_1} W^1_\tau \right] \right] \\
& = \mathbb{E}^2\left[ \mathbbm{1}_{A_2} \mathbb{E}^1\left[ \mathbbm{1}_{A_1} \mathbb{E}^1\left[ W^1_\tau \big| \mathcal{F}^1(0) \right] \right]\right].
\end{split}
\end{equation}
Observe that, while $\tau$ (begin an $\mathbb{H}^1_\pi$-stopping time) may have a dependence on $\Omega^2$, inside the iterated integral with the dependence on $\Omega^2$ fixed, it is an $\mathbb{F}^i$-stopping time. Hence, as an application of Eq. \eqref{eqn:f-eqn}, we have the bound
\begin{equation}
\begin{split}
\mathbb{E}\left[ \mathbbm{1}_A W^1_\tau \right] & = \mathbb{E}^2\left[ \mathbbm{1}_{A_2} \mathbb{E}^1\left[ \mathbbm{1}_{A_1} \mathbb{E}^1\left[ W^1_\tau \big| \mathcal{F}^1(0) \right] \right]\right] \leq \mathbb{E}^2\left[ \mathbbm{1}_{A_2} \mathbb{E}^1\left[ \mathbbm{1}_{A_1} 0 \right]\right] = 0.
\end{split}
\end{equation}
Hence, for all rectangles $A \in \mathcal{H}^1_\pi(0)$, $\mathbb{E}\left[ \mathbbm{1}_A \mathbb{E}\left[ W^1_\tau \big| \mathcal{H}^1_\pi(0) \right]\right] \leq 0$. This extends via the usual monotone-class type argument to \textit{all} $A \in \mathcal{H}^1_\pi(0)$. Hence, it follows that for all $\tau \in \hat{\mathbb{H}}^1_\pi(0)$,
\begin{equation}
\mathbb{E}\left[ W^1_\tau \big| \mathcal{H}^1_\pi(0) \right] \leq 0  \mbox{ ($\mathbb{P}$-a.e.)}.
\end{equation}
This establishes the result.\qed

 The proof  of   Proposition \ref{prop2} below requires the following technical lemma. Its proof   follows along the lines of the proofs of  Theorems 4.1 - 4.3    in \cc{snell1952}, see also the Optimal Optional Stopping Lemma in \cc{derman1960replacement}.

\begin{lemma}\label{lem:optional}
In an arbitrary probability space with a filtration $\mathbb{J} = \{ \mathcal{J}_t \}_{t \geq 0}$, consider an adapted discrete-time process $\{Z_t\}_{t \geq 0}$ such that $\mathbb{E} \left[ \sup_{\mathbb{N}} \lvert Z_t \rvert \big| \mathcal{J}_0 \right] < \infty$. If the $\mathbb{J}$-stopping time $\tau^* \in \hat{\mathbb{J}}(0)$ defined by
\begin{equation}
\tau^* = \inf \{ n > 0 : \underset{ \tau \in \hat{\mathbb{J}}(n) }{\esssup} ~\mathbb{E} \left[ Z_\tau \big| \mathcal{J}_n \right] \leq Z_n  \}
\end{equation}
is almost surely finite, then
\begin{equation}
\mathbb{E} \left[ Z_{\tau^*} \big| \mathcal{J}_0 \right] = \underset{ \tau \in \hat{\mathbb{J}}(0) }{\esssup} ~\mathbb{E} \left[ Z_\tau \big| \mathcal{J}_0 \right] \mbox{ ($\mathbb{P}$-a.e.).}
\end{equation}
\end{lemma}

 {\bf Proof of Proposition} \ref{prop2}.
Recall that we need to show that  for any time $t_0 < \sigma^i$, there exists a $\tau \in  \hat{ \mathbb{F}}^i  (t_0)$ such that $\rho^i(t_0) = \rho^i(t_0, \tau) \mbox{ \emph{($\mathbb{P}^i$-a.e.)}}$.

We have that for all $\hat{\tau} \in \hat{\mathbb{F}}^i(t_0)$, $\rho^i(t_0, \hat{\tau}) \leq \rho^i(t_0) \mbox{  ($\mathbb{P}^i$-a.e.)}$. Taking $$\mathbb{P}^i( t_0 < \sigma^i \leq \hat{\tau} \big| \mathcal{F}^i(t_0)) = \mathbb{E}^i \left[ \mathbbm{1}_{ \{ t_0 < \sigma^i \leq \hat{\tau}\} } \big| \mathcal{F}^i(t_0) \right],$$ we have in parallel with Eq. \eqref{eq:prop21},
\begin{equation}\label{eqn:0-upper-bound}
\mathbb{E}^i \left[ X^i_{ \sigma^i \wedge \hat{\tau} } - X^i_{t_0} - \rho^i(t_0) \mathbbm{1}_{ \{ t_0 < \sigma^i \leq \hat{\tau}\} } \big| \mathcal{F}^i(t_0) \right] \leq 0 \mbox{  ($\mathbb{P}^i$-a.e.).}
\end{equation}

Defining
\begin{equation}
\epsilon = - \underset{\hat{\tau} \in  \hat{ \mathbb{F}}^i (t_0) }{ \esssup} ~ \mathbb{E}^i \left[ X^i_{ \sigma^i \wedge \hat{\tau} } - X^i_{t_0} - \rho^i(t_0) \mathbbm{1}_{ \{ t_0 < \sigma^i \leq \hat{\tau}\} } \big| \mathcal{F}^i(t_0) \right],
\end{equation}
we have that $\epsilon \geq 0 \mbox{  ($\mathbb{P}^i$-a.e.)}$. We may use $-\epsilon$ as an improved upper bound in \eqqref{eqn:0-upper-bound}. This may be rearranged to yield
\begin{equation}
\rho^i(t_0, \hat{\tau}) \leq \rho^i(t_0) - \frac{ \epsilon}{ \mathbb{E}^i \left[  \mathbbm{1}_{ \{ t_0 < \sigma^i \leq \hat{\tau}\} }  \big| \mathcal{F}^i(t_0) \right] } \leq \rho^i(t_0) - \epsilon \mbox{  ($\mathbb{P}^i$-a.e.)}.
\end{equation}

Since the above property holds for all such $\hat{\tau}$, it extends to the essential supremum, yielding
\begin{equation}
\rho^i(t_0) \leq \rho^i(t_0) - \epsilon \mbox{  ($\mathbb{P}^i$-a.e.)},
\end{equation}
or equivalently that $\epsilon \leq 0  \mbox{  ($\mathbb{P}^i$-a.e.)}$. In conjunction with the first observation, that $\epsilon \geq 0 \mbox{  ($\mathbb{P}^i$-a.e.)}$,  we have  $\epsilon = 0 \mbox{  ($\mathbb{P}^i$-a.e.)}$, i.e., 
\begin{equation}\label{eqn41}
\underset{\hat{\tau} \in  \hat{ \mathbb{F}}^i (t_0) }{ \esssup} ~ \mathbb{E}^i \left[ X^i_{ \sigma^i \wedge \hat{\tau} } - X^i_{t_0} - \rho^i(t_0) \mathbbm{1}_{ \{ t_0 < \sigma^i \leq \hat{\tau}\} } \big| \mathcal{F}^i(t_0) \right] = 0 \mbox{  ($\mathbb{P}^i$-a.e.)}.
\end{equation}

Define $Z^i_t = X^i_{ \sigma^i \wedge t } - X^i_{t_0} - \rho^i(t_0) \mathbbm{1}_{ \{ t_0 < \sigma^i \leq t\} }$. Note that the integrability condition of Lemma \ref{lem:optional} is satisfied due to Eq. \eqref{eqn:integrability}. For $t \geq \sigma^i$, $Z^i_t$ is constant, hence $\tau^* \leq \sigma^i < \infty$ almost surely. Hence we may apply Lemma \ref{lem:optional} here to yield a stopping time $\tau^* \in \hat{\mathbb{F}}^i(t_0)$ such that
 \begin{equation}
\mathbb{E}^i \left[ X^i_{ \sigma^i \wedge \tau^*} - X^i_{t_0} - \rho^i(t_0) \mathbbm{1}_{ \{ t_0 < \sigma^i \leq \tau^* \} } \big| \mathcal{F}^i(t_0) \right] = 0 \mbox{  ($\mathbb{P}^i$-a.e.)},
\end{equation}
or
\begin{equation}
\rho^i(t_0) = \frac{ \mathbb{E}^i \left[ X^i_{ \sigma^i \wedge \tau^*} - X^i_{t_0} \big| \mathcal{F}^i(t_0) \right] }{ \mathbb{P}^i \left( t_0 < \sigma^i \leq \tau^*  \big| \mathcal{F}^i(t_0) \right) } = \rho^i(t_0, \tau^*)\mbox{  ($\mathbb{P}^i$-a.e.)}.
\end{equation}

Hence, the solo-payout index $\rho^i(t_0)$ is realized ($\mathbb{P}^i$-a.e.) for some $\mathbb{F}^i$-stopping time $\tau^* > t_0$. \qed

 {\bf Proof of Proposition \ref{prop:gittins-decreasing}}.
 
For $k > 0$, let $\tau^i_k < \sigma^i$, and therefore $\tau^i_{k-1} < \sigma^i$. Defining  $$Z^i_t = X^i_{ \sigma^i \wedge t } - X^i_{\tau^i_{k-1}} - \rho^i(\tau^i_{k-1}) \mathbbm{1}_{ \{ \tau^i_{k-1} < \sigma^i \leq t\} },$$
note that for $t > \tau^i_k$: $Z^i_t - Z^i_{\tau^i_k} = X^i_{ \sigma^i \wedge t } - X^i_{\tau^i_k} - \rho^i(\tau^i_{k-1}) \mathbbm{1}_{ \{ \tau^i_{k} < \sigma^i \leq t\} }$.

It follows from the proof of Proposition \ref{prop2} that the solo-payout index from time $\tau^i_{k-1}$ is realized by a $\tau^i_k$ such that
\begin{equation}\label{eqn:34?}
\underset{ \tau' \in \hat{\mathbb{F}}^i (\tau^i_{k}) }{ \esssup } ~\mathbb{E}^i \left[ Z^i_{\tau'} \big| \mathcal{F}^i( \tau^i_{k}) \right] \leq  Z^i_{\tau^i_k} \mbox{  ($\mathbb{P}^i$-a.e.)},
\end{equation}
or
\begin{equation}
\underset{ \tau' \in \hat{\mathbb{F}}^i (\tau^i_{k}) }{ \esssup } ~\mathbb{E}^i \left[ X^i_{ \sigma^i \wedge \tau' } - X^i_{\tau^i_k} - \rho^i(\tau^i_{k-1}) \mathbbm{1}_{ \{ \tau^i_{k} < \sigma^i \leq \tau'\} } \big| \mathcal{F}^i( \tau^i_{k}) \right] \leq  0 \mbox{  ($\mathbb{P}^i$-a.e.)}.
\end{equation}
From the above, for any $\tau' \in \hat{\mathbb{F}}^i(\tau^i_k)$, we have
\begin{equation}
\frac{ \mathbb{E}^i \left[ X^i_{ \sigma^i \wedge \tau' } - X^i_{\tau^i_k} \big| \mathcal{F}^i( \tau^i_{k}) \right] }{ \mathbb{P}^i \left( \tau^i_{k} < \sigma^i \leq \tau' \big| \mathcal{F}^i( \tau^i_{k}) \right) } \leq \rho^i(\tau^i_{k-1})  \mbox{  ($\mathbb{P}^i$-a.e.)}.
\end{equation}
Taking the essential supremum over such $\tau'$ establishes that $\rho^i(\tau^i_{k}) \leq \rho^i(\tau^i_{k-1})$, $ \mbox{  ($\mathbb{P}^i$-a.e.)}$. \qed
%
 

\medskip

{\bf Acknowledgements.}  
We acknowledge support for this work from the National Science Foundation, NSF grants: CMMI-1662629 and CMMI-1662442.

\end{document}